\documentclass[11pt]{article}
\usepackage[utf8]{inputenc}
\usepackage[margin=1in]{geometry}

\usepackage{amsmath,amsthm,amsfonts,amssymb,mathdots,array,mathrsfs,bm,bbm,stmaryrd,graphicx,subfigure,xcolor}
\usepackage{microtype}
\usepackage{graphicx}
\usepackage{subfigure}
\usepackage{booktabs} 
\usepackage{natbib}
\setcitestyle{open={(},close={)}}
\usepackage[colorlinks=true,citecolor=blue]{hyperref}

\usepackage{multicol}

\usepackage[ruled,vlined]{algorithm2e}

\usepackage{amsmath}
\usepackage{amsfonts}
\usepackage{amssymb}
\usepackage{mathbbol}
\usepackage{graphicx}
\usepackage{fullpage}
\usepackage{url}
\usepackage{csquotes}
\usepackage[american]{babel}
\usepackage{comment}
\usepackage{multirow}
\usepackage{url}
\usepackage{times,subfigure}

\newtheorem{theorem}{Theorem}
\newtheorem{lemma}[theorem]{Lemma}

\newtheorem{corollary}{Corollary}

\newtheorem{mydef}{Definition}

\usepackage{xcolor}

\begin{document}

\title{\huge Towards Constructing Finer then Homotopy Path Classes}

\author{\vspace{0.5in}\\\textbf{Weifu Wang} \  and \ \textbf{Ping Li} \\\\
Cognitive Computing Lab\\
Baidu Research\\
%No. 10 Xibeiwang East Road, Beijing 100193, China \\
10900 NE 8th St. Bellevue, WA 98004, USA\\\\
  \texttt{\{harrison.wfw,  pingli98\}@gmail.com}
}

\date{\vspace{0.5in}}
\maketitle

\begin{abstract}\vspace{0.3in}

\noindent This work presents a new path classification criterion to distinguish paths geometrically and topologically from the workspace, which is divided through cell decomposition, generating a medial-axis-like skeleton structure. We use this information as well as the {\em topology of the robot} to {\em bound} and {\em classify} different paths in the configuration space. We show that the path class found by the proposed method is equivalent to and finer than the path class defined by the homotopy of paths. The proposed path classes are easy to compute, compare, and can be used for various planning purposes. The classification builds heavily upon the topology of the robot and the geometry of the workspace, leading to an alternative fiber-bundle-based description of the configuration space. We introduce a planning framework to overcome obstacles and narrow passages using the proposed path classification method and the resulting path classes. 
\end{abstract}

\newpage

\section{Introduction}

Let us consider the following scenario: a box in the middle of the room, what are the different ways to pass the box and reach the other side of the room? We can pass the box from the left or the right; if the box is not very tall, we can also jump over the box. We recognize that these are different paths, but it is difficult to define their differences topologically. If we are not allowed to jump over the box, i.e., viewing navigating the room as a planar problem, the two paths passing the box from left or right belong to two different homotopy classes. In other words, one cannot find a way to continuously {\em deform} the path from one to the other without passing through the box. Path classes are helpful when planning for multiple robots or finding diverse paths to increase the chance of survival if different paths fail with various probabilities. 

In the original 3D problem, because there is a possibility of jumping over the not-so-tall box, the paths passing through the left and right belong to the same homotopy class. In recent years, topological tools such as homotopy and homology were introduced to categorize paths~\citep{bhattacharya2013invariants}, but such definition becomes {\em weak} in dimensions beyond two unless the obstacles create a {\em topological hole} in the space. In our toy example, the box has to be touching the roof to create such a {\em topological hole}. Another example is shown in Figure~\ref{fig:env_diff_paths}, where there are two pillars with two platforms attached to the obstacles. Five colored paths are geometrically distanced, but only one (the bright blue) path belongs to a different homotopy class if the tall pillar touches the roof. One of the primary purposes of this work is to mathematically distinguish the above five paths and show the derived properties extending to the configuration space.

 \begin{figure}[h]
 \vspace{0.2in}
    \centering
    \includegraphics[width=3in]{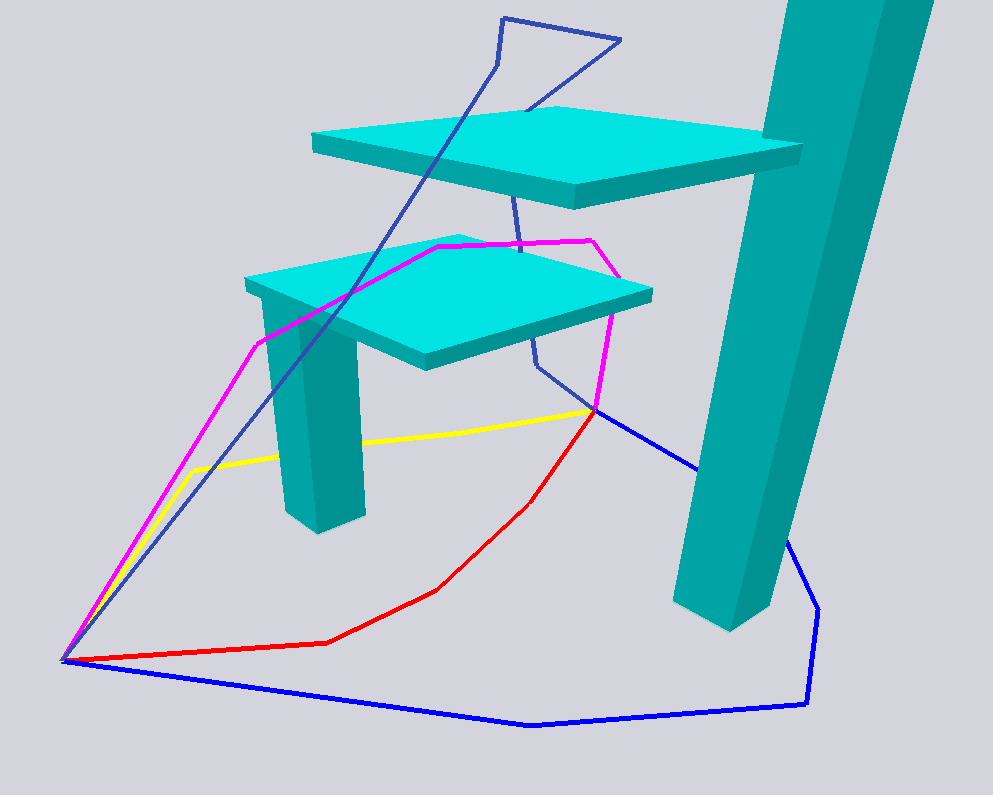}
    \caption{A set of paths that belong to the same homotopy class except the bight blue path closest to the viewpoint. }
    \label{fig:env_diff_paths}
 \end{figure}

% approaches

We propose a new path categorization criterion, so that different path classes can either mean topologically or geometrically different. The classification starts with analyzing the workspace rather than the configuration space. We show that many critical workspace properties, including homotopy information, can extend to the configuration space. A workspace cell-decomposition is performed to find the almost-Voronoi cells-complement, which contains the skeleton or medial-axis \citep{xu1992motion} of the workspace. We show that though it starts in the workspace, the resulting path class properties hold in the configuration space. 

\newpage

We should mention that, although some toy examples and analysis will be presented in $\mathbb{R}^2$, the method is most beneficial in 3D. Unlike the methods that conduct configuration space analysis and use sampling-based methods, our proposed path categorization standard only depends on the topological and geometrical properties of the workspace and robot itself, which are both easy to analyze and visualize. We show that the proposed criterion can distinguish configuration space paths with different workspace properties, and the resulting path classes can be used for various planning tasks, like obtaining path existence~\citep{mccarthy2012proving}, finding diverse paths~\citep{lyu2016k}, and planning for multiple agents.

% Using path classes to build a new path planning framework. 
A new planning framework can be built upon the proposed path classes and associated methods. The framework aims to separate topological and geometrical planning into two stages. In the workspace, the path classes found can represent the connectivity information in both the workspace and the configuration space. We can integrate the geometrical constraints into the topological paths and generate satisfying geometrical paths using the topological paths found using the path classes. In this work, we focus on the theoretical foundations of path classification and discuss the geometrical interpolation step.

% Below are too detailed. Now discuss the differences and advantages.

The proposed path classification starts from the workspace. The first step is to distinguish the free regions from the obstacles. Since we know the geometries of the obstacles in the workspace, we can use cell-decomposition as the first step of the processing. We adapt a Delaunay-triangulation-based cell decomposition, producing regions with properties similar to $\alpha$-shapes (also written as alpha-shapes)~\citep{kreveld2011shape,mccarthy2012proving}. The primary focus of this work is that we show the cell decomposition in the workspace, plus the {\em topology of the robot} information, can correctly {\em approximate} the simplicial complex in the configuration space. 
%, having the same type of path connectivity.

Cell-decomposition is a well-studied technique in the planning community, and many variations have been used in the past. Axis-aligned cell decomposition is simple, but the curse of dimensionality arises when conducted in the configuration space. Alpha-shape and simplex have also been used in planning~\citep{mccarthy2012proving}, but the method usually requires some dense sampling in the configuration space. Medial-axis approaches often employ cell-decomposition methods as well. 

The topology of the robot is the abstraction of the possible geometries a robot can reach and can usually be described using a collection of chains. For example, an open-chain has a different topology compared to a closed chain, but two open chains with different link lengths share the same topology. Paths within the same proposed path class have good geometric properties for a serial chain. The paths for a collection of chains can inform a path for the complex robot. We prove that the proposed categorization generates the same topology path classes when there are topological holes but can also further classify additional path classes based on the geometrical information when topological holes do not exist. Our path classification is {\em finer} than the homotopy class. 

As presented in this paper, we are only considering polyhedron obstacles. While the proposed algorithm can handle concave obstacles, the decomposition can be expensive when there are many closely placed concave obstacles. This paper is primarily theoretical, with proof-of-concept examples and experiments. Another limitation of the current work is that the topological holes within the workspace obstacles are not well analyzed. We simply categorize such holes as regions adjacent to a single obstacle. 

\newpage

\textbf{Our main contributions} in this work are summarized as follows:
\begin{enumerate}
    \item We introduce a {\em finer} path classification method compared to pure topological-based tools; our method can further identify paths that are not distinguishable topologically but are still {\em distant} from each other; 
    \item The path class provides high-level topological paths; we present definitions, proofs, and algorithms to construct the topological paths; 
    \item  We introduce a planning framework based on the proposed path classes, which can provide good kinematic path candidates within tight spaces; 
    \item We show that the proposed methods are robust to small perturbations of the obstacles and can make updates efficiently; 
    % In the future, add the framework to the probabilistic robotics domain
\end{enumerate}

% HERE, to add
% Alpha shape has one major flaw: it does not deal with obstacles with resolutions well and requires a hyper-parameter. So, a more straightforward approach would work for the proposed method. 
% experiments show a few classes of paths and their corresponding configuration space example. 
% if possible, show difference or runtime difference to the alpha shape

\section{Related Work}

Researchers have been using homotopy class to categorize paths on the plane and in space~\citep{kim2012trajectory,bhattacharya2012search}. In 2D, the homotopy class~\citep{hatcher2022algebraic} of a path perfectly categorizes the paths that wrap around the planar obstacles differently. Based on such topological tools, works on how to search paths~\citep{bhattacharya2012search} and find manipulation approaches~\citep{bhattacharya2015topological} using homotopy classes were proposed. More recent work has studied how to extend the ideas into 3D around knot-like obstacles and attempts to combine the path class constraints with path optimality constraints~\citep{bhattacharya2018path}. The shortest path is one of the ultimate goals for motion planning, but only limited cases have been solved, especially when the obstacles are involved. In~\citet{grigorie1998polytime} and~\citet{hershberger1994computing}, the authors studied how to find the shortest paths in a given homotopy~class. In~\citet{pokorny2016topological}, the authors proposed to use sampling-based methods to compute filtrations and persistent homology to classify paths. 

In this and many other existing works, the authors attempt to distinguish paths by measuring the persistent homology of obstacles~\citep{pokorny2016topological} or finding homotopy classes~\citep{bhattacharya2015topological}. However, in many cases, especially in higher dimensions, the topological holes are hard to create or compute. For example, in Figure~\ref{fig:env_diff_paths}, the two pillars and the connected platforms do not create a topological hole, all paths are homotopy equivalent, and the persistent homology of the obstacles is of low dimensions. The persistent homology and associated measure may not distinguish the path classes fully. 

At the same time, researchers have also attempted to define path diversities beyond the topological properties. In~\citet{branicky2008path}, the authors introduced path diversity by decomposing space into cells and used the cell information to classify paths. Path diversity definitions also exist based on the measurement of closeness and relevance to the goal~\citep{knepper2009path,knepper2011thesis,knepper2012toward}. Researchers have also introduced probability into the definitions of path diversity. Some researchers have defined {\em survivability} as a critical concept in the path diversity, measuring how likely paths are to succeed if some probabilistic events were to happen at some locations~\citep{erickson2009survivability,lyu2016k}. In addition, researchers have studied path diversity in various realistic settings, such as ambush avoidance~\citep{boidot2015optimal}, and detection avoidance for aircraft~\citep{zabarankin2006aircraft}, and how to navigate through mined waters~\citep{babel2015planning}. In~\citet{schwartz1983piano}, the authors presented a semi-algebraic description to distinguish various obstacles and can be used to define path classes. However, the sets can become increasingly difficult to compute when the obstacles become complex. In~\citet{simeon2004manipulation}, the authors also proposed to use sampling-based methods and reachable regions to perform motion planning, which is  similar to the method proposed in this work. One of the main differences is that we are not using sampling-based approaches in our proposed method, only the obstacle descriptions and robot structure. 

In networking, the study of path diversity has also attracted attention, as the network communication and information flow needs to be diverse to handle traffic conditions~\citep{yuan2005minimum}. Heuristic measurements were introduced to find diverse paths~\citep{voss2015heuristic} and compute the maximum diversity of paths~\citep{marti2013heuristics}. Similar to the graph coloring problem, researchers also attempted to color different paths to define their diversity~\citep{chan2014multi}. Geometric methods such as visibility have also been used to generate diverse paths~\citep{quispe2013generation}, or cover spaces like art gallery problems~\citep{orourke1987art}. Theoreticians have also attempted to analyze how to find disjoint paths and the related complexity~\citep{xu2006complexity}. Path diversity has also been studied under pursuit and evasion~\citep{chung2011search}. 

Multi-robot planning is one of the most common scenarios associated with path diversity in robotics. Multi-robot coordination usually needs diverse paths to succeed~\citep{parker2009multiple,pecora2018loosely}. Multi-robot planning has also been studied on graphs~\citep{pereyra2017path}. In more challenging scenarios, people have studied the path diversity problem in online settings, such as diversity maintenance problem~\citep{curkovic2020diversity}, control problems~\citep{liu2019homotopy}, and online planning and trajectory smoothing and optimization~\citep{zhu2015convex}.

\section{Decomposition and Robot Topology}

We introduce the foundations of the path classification: workspace decomposition and robot topology representation. The following section shows the product relation between workspace cells and the robot topology that can successfully reproduce the path classes in the configuration space. We decompose the workspace based on the Voronoi boundaries' dimension and use Delaunay triangulation as the computation basis. 

% Change from  hull -> joint-hull
\begin{mydef}
Let there be two disjoint objects $\mathcal{O}_i$ and $\mathcal{O}_j$, denote their convex hull as $CH(\mathcal{O}_{i, j})$. Define a joint-hull $JH(\mathcal{O}_{i, j}) = X_{i, j}$ as a subset of $CH(\mathcal{O}_{i, j})$ so that $X_{i, j}\cap X_{s, t} = \emptyset$ for $i\neq j\neq s\neq t$, and $X_{i, j}\cap X_{i, k} = \mathcal{O}_i$ for $i\neq j, i\neq k, j\neq k$. Define the joint cover of a given workspace $W$ with obstacles $\mathcal{O}_i, 0 < i \leq N$ as the union of all joint hulls for all pairs of objects / obstacles in a given space, $JC(W) = \cup{X_{i, j}}, \forall i, j\in N, i\neq j$. 
\label{def:jhjc}
\end{mydef}

In other words, a joint hull between two obstacles $\mathcal{O}_i$ and $\mathcal{O}_j$ contains Voronoi cell boundaries that have equal distances to $\mathcal{O}_i$ and $\mathcal{O}_j$. We classify free regions based on the number of adjacent obstacles. The degree of Voronoi cell boundaries encodes the degree of holes the corresponding region would create when removed from the workspace. 

% Relation to alpha shapes
The $\alpha$-shape is a one-parameter family of polytopes that quantity the {\em shape} of a point set (parameterized by $\alpha\in\mathbb{R}^+$). When $\alpha = 0$, the alpha shape of $\mathcal{S}$ is simply $\mathcal{S}$ itself. When $\alpha = \infty$, we have that the alpha shape of S is equivalent to the convex hull of $\mathcal{S}$. For intermediate $\alpha$, the $\alpha$-shape is more complicated but describes the shape of $\mathcal{S}$ in various levels of detail. The $\alpha$-shapes provide a natural interpretation for a union of balls and are closely related to {\u{C}}ech complex~\citep{ghrist2014elementary}. 

Using $\alpha$ shapes to describe the obstacles is an excellent way to describe some objects' convex and concave regions at some resolution. Finding the $\alpha$-shape description for the obstacles is usually preferred. In~\citet{mccarthy2012proving}, the authors used sampling-based methods in the configuration space to build $\alpha$-shapes and approximate the configuration obstacles, then use the resulting information to predict the existence of paths.

The joint hull of two obstacles is an $\alpha$-shape and is a subset of the triangulation between the two obstacles in the workspace. The description may not correspond to a fixed $\alpha$. So, in other words, the joint hull is a weighted $\alpha$-shape~\citep{edelsbrunner1992weighted}. It is challenging to select appropriate weights to describe the desired details of the geometry. The known runtime for finding the weighed $\alpha$-shapes is $O(n^2)$ in $\mathbb{R}^3$ for a given weight. Because the weight is hard to select, we can find the joint hull another way, based on triangulation. We can avoid directly finding the weighted-$\alpha$-shape because there can be many valid joint hulls for $\mathcal{O}_i$ and $\mathcal{O}_j$. 

We use Delaunay triangulation to construct a joint cover. Based on the Definition~\ref{def:jhjc}, there can be more than one valid joint hull between two objects. We would like to find the Joint Cover that maximizes the region adjacent to two obstacles.

Let us introduce Algorithm~\ref{algorithm:joint-cover} to construct the joint cover. The process builds upon the Delaunay Triangulation and then merges adjacent regions based on the number of adjacent $\mathcal{O}$. The algorithm's runtime is worst-case $O(n^2)$ for a total of $n$ vertices among all obstacles, which is the worst-case bound for triangulation. However, just as the triangulation and $\alpha$-shape constructions, the practical runtime is usually much less than the theoretical worst case. The construction can also be incremental, and runtime can be faster if the obstacles and vertices are sorted geometrically first. Figure~\ref{fig:env_box_fig} shows some Joint Hulls among random 2D polygon obstacles. Figure~\ref{fig:env-3d} shows the joint hulls in 3D for an environment with multiple obstacles.

\begin{figure}[h]
\begin{center}
\mbox{
    \includegraphics[width=2in]{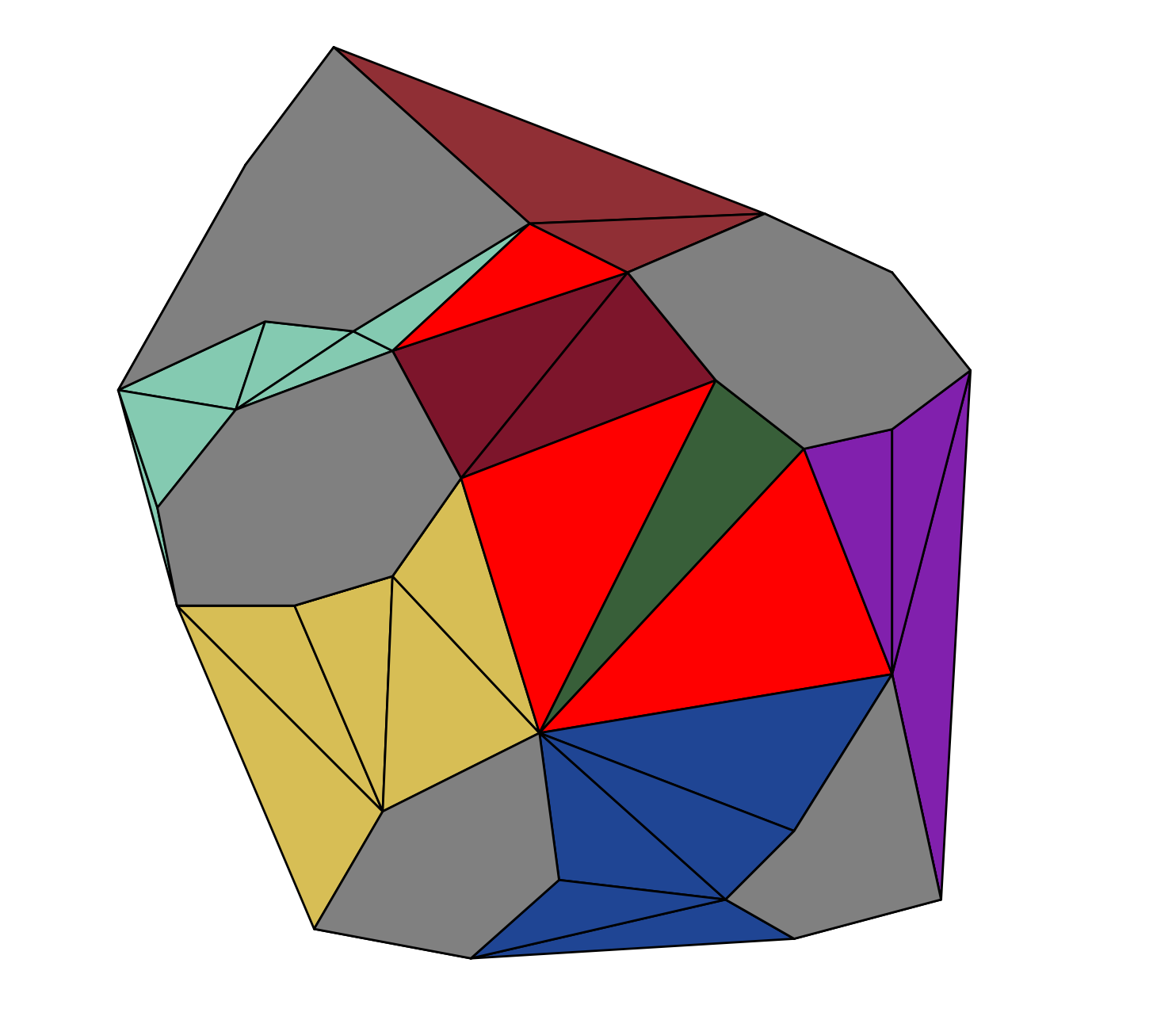}
    \includegraphics[width=1.5in]{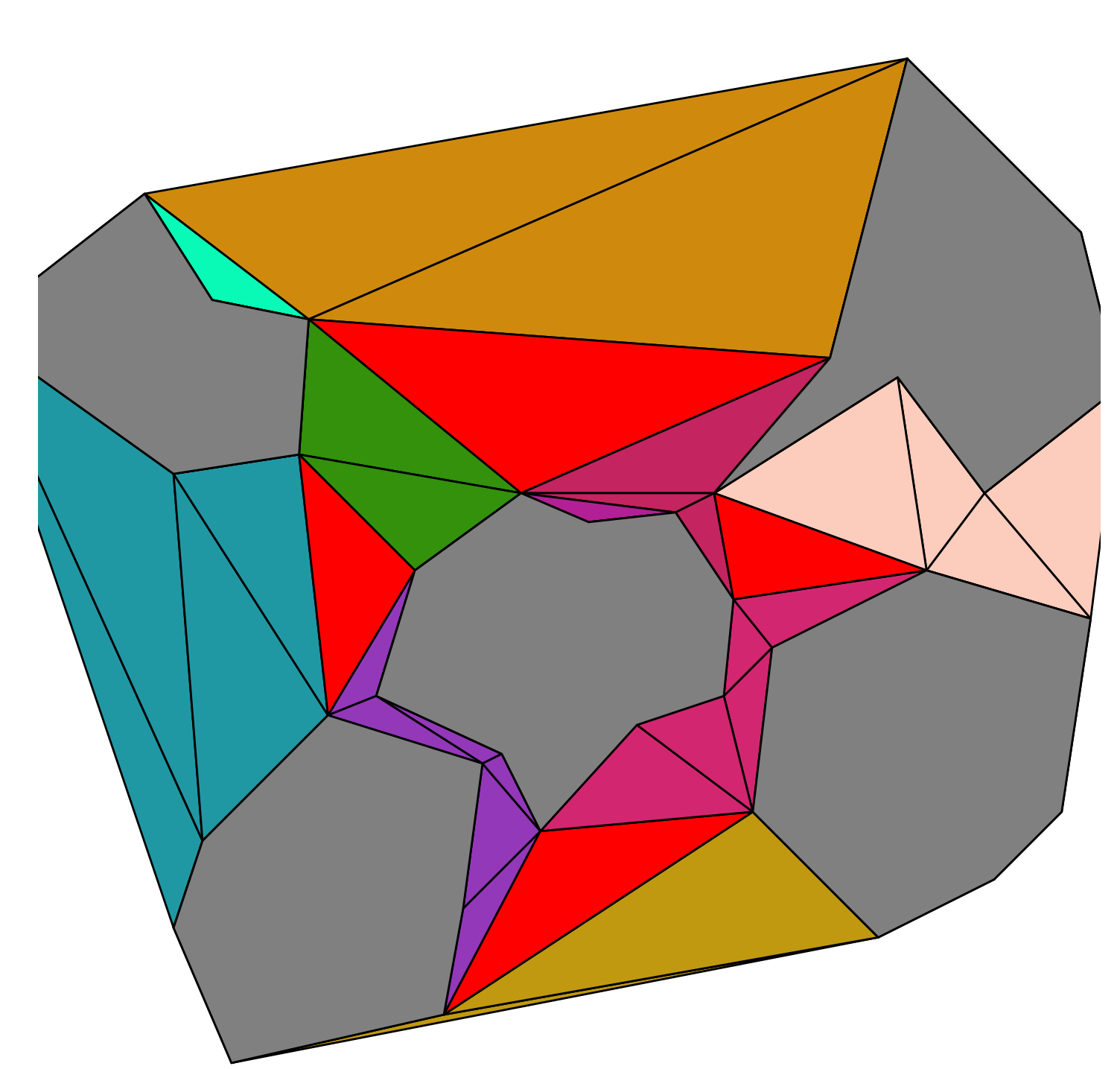}
}
\vspace{-0.1in}

\caption{Random gray polygons and corresponding Joint Hulls obtained using Algorithm~\ref{algorithm:joint-cover}.}
\label{fig:env_box_fig}\vspace{-0.1in}
\end{center}
\end{figure}

\begin{figure}[h]
\centering 
\subfigure[3D Environment]{\includegraphics[height=2in]{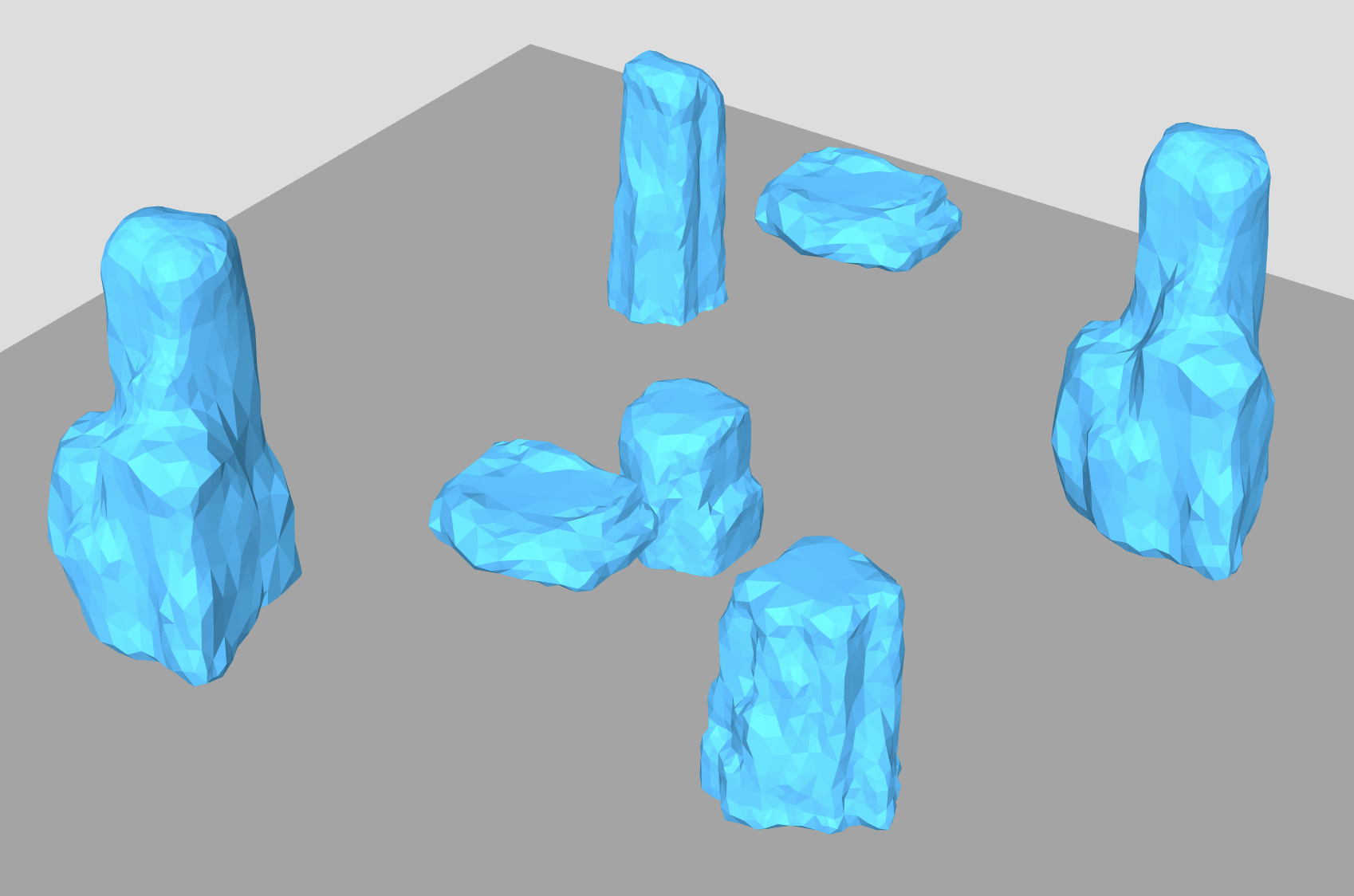}}
\subfigure[Joint hulls]{\includegraphics[height=2in]{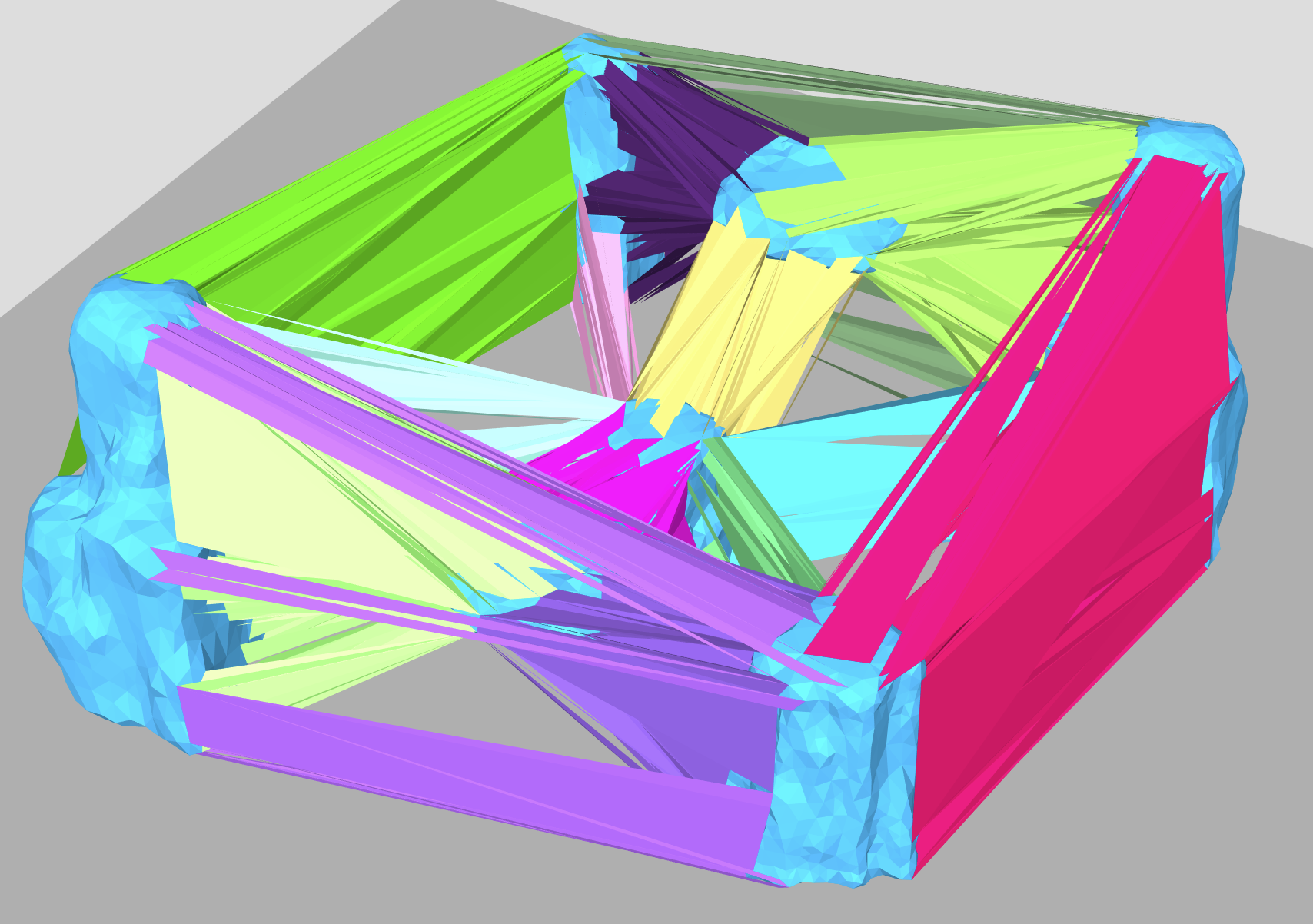}}

\vspace{-0.1in}

\caption{A 3D environment and boundaries of the joint hulls between adjacent obstacles. In the figure on the write, different colored areas represent regions adjacent to two obstacles, and the spaces among the colored regions are the common spaces that are adjacent to multiple objects. }
\label{fig:env-3d}\vspace{-0.4in}
\end{figure}

\newpage

\begin{algorithm}[t]
\SetAlgoLined
\SetKwInOut{Input}{input}\SetKwInOut{Output}{output}
\Input{$\mathcal{O}_i, i=1,\ldots, N$, $\mathcal{W}\in\mathbb{R}^d$}
\Output{$JC(\mathcal{W})$, $G_{\mathcal{A}}$}
$DT\leftarrow$ Delaunay triangulation of obstacles\;
Initialize $G_{\mathcal{A}} = (V, E)$ where $V = \mathcal{O}_{\cdot}$ and $E = \emptyset$\;
% $\mathcal{J}_{i, j}\leftarrow\emptyset, \forall\mathcal{O}_i, \mathcal{O}_j$\; 
\While{DT is not empty} {
$s_i\leftarrow $ pop triangle from DT\;
$L\leftarrow$ the $\mathcal{O}$ $s_i$ is adjacent to\;
$\mathcal{N}\leftarrow $ neighboring triangles of $s_i$\;
\While{$\mathcal{N}\neq\emptyset$} {
    $s_k\leftarrow$ pop $\mathcal{N}$\;
    \If {current region has the same adjacent $\mathcal{O}$ as $s_k$}{
        Add neighbors of $s_k$ into $\mathcal{N}$\;
        Remove $s_k$ from DT\;
        Merge $s_k$ into current region\;
        Update current region neighboring $\mathcal{O}$ set\;
    } 
}
Set the merged region as $\mathcal{J}(\cdot)$ based on the adjacent $\mathcal{O}_\cdot$\;
}
\Return $JC\leftarrow \cup \mathcal{J}(\cdot)$, $G_\mathcal{A}$
 \caption{Joint Cover}
 \label{algorithm:joint-cover}
\end{algorithm}

The adjacency is closely related to visibility. If two obstacles are mutually visible, a free region must exist adjacent to both obstacles. Nevertheless, this free region may also be adjacent to other obstacles. We can use visibility to construct $G_\mathcal{A}$ quickly. Visibility between obstacles is hard to compute directly. The triangulation and merging process can find $G_\mathcal{A}$ as well as $JC$ simultaneously. 

% Robot topology
The configuration space is closely tied to the topology of the robot. A robot $\mathcal{B}$ can deform into different geometries and sometimes even different topologies. Embedding various robot deformations into the workspace will result in different configurations. Let us select a set of {\em key points} on the robot, including but not limited to the endpoints of each joint on the robot. We can show that embedding the {\em key points} into the workspace can fully represent the embedding of robot deformations and infer the associated configuration.

% points or chain? 

Let us select $k$ {\em key points} ($p_i = (x_i, y_i, z_i)$, $i\in\{1, 2, \ldots, k\}$) on the robot with fixed pair-wise distances between adjacent points. Note, a single key-point $p$ may be adjacent to many key points ($A_p$), and it has a fixed distance to each adjacent key point $q\in A_p$, but two points in $A_p$ may not have a fixed distance constraint. For example, $p$ can be the right shoulder of a humanoid robot. The set $A_p$ includes the spine connecting to the head, the left shoulder, and the right elbow. These points in $A_p$ all have fixed distances to the right shoulder, but they do not need to maintain any fixed distances. The relations among these key points are fixed for a given robot. Here, we do not consider the robots that can {\em reconfigure} themselves.

When the robot is a single point, the planning can be simple, and the cells derived above can give the path classes. When the robots are more complex, the path classes in the configuration space can differ from the workspace cells. We show that we can combine the workspace cells and the topology of the robot information and map to the path classes in the configuration spaces. This way, we can independently analyze the workspace information and robot topology without analyzing or sampling in the high-dimensional and hard-to-describe configuration space. 

\newpage

The configuration space is a simplified representation of the workspace information combined with the robot deformations, so that configuration space $\mathcal{C}$ is homeomorphic to $W\times P(\mathcal{B})$. In fact, $(\mathcal{C}, \pi, B, P(\mathcal{B}))$ is a fiber bundle~\citep{hatcher2022algebraic} if $\pi$ maps a configuration to its workspace embedding region. The configuration space has the subspace topology of $W\times P(\mathcal{B})$. Because the non-trivial mapping of $\pi$ is caused by the obstacles, where the paths are disconnected, a path in $W\times P(\mathcal{B})$ corresponds to a path in $\mathcal{C}$ if it is valid in $\mathcal{C}$. We can define valid path classes on $W\times P(\mathcal{B})$ to classify paths in $\mathcal{C}$. The path in $W$ and the embedding function $\pi$ would be easy to visualize and analyze independently. 

Given $W$ and $JC(W)$, let us label the free regions to represent the workspace connectivity better. Note that a free region is adjacent to two obstacles if it is a compact region inside a joint hull. Other regions can be found using Algorithm~\ref{algorithm:joint-cover}. The free regions and their adjacency form a simplicial complex representation of the workspace, and their adjacency relation $G_{\mathcal{A}}$ symbolizes the workspace connectivity. 
\begin{mydef}
Define the label of a free-region as: 
\begin{itemize}% [leftmargin=*]
\setlength\itemsep{-0.5pt}
    \item If a free-region is adjacent to a single obstacle $\mathcal{O}_i$ and not a hole within obstacle $\mathcal{O}_i$, we assign label $i$ to the region; 
    \item If a free-region is a hole within obstacle $\mathcal{O}_i$, we assign label $-i\times k$ to the region, for the $k$th hole; 
    \item If a free-region is adjacent to $\mathcal{O}_i$ and $\mathcal{O}_j$, assign label $(N*i + j)$ assuming $j < i$; 
    \item If a free-region is adjacent to more than two obstacles, $\{\mathcal{O}_i\}$, $i\in P$, where $P$ is a subset of positive integers no larger than $N$; assign label $f(\cdot)$ to the region, where $f$ is the G\"{o}del numbering function, $f(i_1, i_2, \ldots, i_{|P|}) = 2^{i_1}\cdot 3^{i_2}\cdot 5^{i_3}\ldots p_{|P|}^{i_{|P|}}$, where $p_j$ is the $j$th prime number, $i_1 > i_2 > \ldots > i_{|P|}$; this encoding is unique for each region.
\end{itemize}
\end{mydef}

The structure of the $k$ key points on a given robot also forms a simplicial complex which does not change even when the robot deforms in most cases. Let us denote the workspace free-region complex as $S_{\mathcal{W}}$, and the robot key-points structural complex as $S_{\mathcal{B}}$. The product relation between the two complex $S_\mathcal{W}\times S_\mathcal{B}$ can be used to represent the embedding of the robot deformation into the workspace. As $S_\mathcal{W}$ and $S_\mathcal{B}$ are abstractions of $W$ and $P(\mathcal{B})$, they are homotopy equivalent to $W$ and $P(\mathcal{B})$, respectively. The product $S_\mathcal{W}\times S_\mathcal{B}$ are homotopy equivalent to $W\times P(\mathcal{B})$. We can therefore have the following result, 
% The path classes on $S_\mathcal{W}\times S_\mathcal{B}$ will be at least topologically equivalent to the paths in $\mathcal{C}$. 

\begin{lemma}
Any two paths that are not homotopy equivalent on $\mathcal{C}$ correspond to two different paths on $S_\mathcal{W}\times S_\mathcal{B}$. 
\label{lemma:equivalence}
\end{lemma}

\begin{proof}
If two paths belong to the same homotopy class in $W\times P(\mathcal{B})$, the corresponding paths must belong to the same homotopy class in $\mathcal{C}$. If there exist two paths belonging to two homotopy classes in $\mathcal{C}$, topological holes exist in $\mathcal{C}$. The hole is either created by workspace obstacle or robot deformation limit. Both will lead to topological holes in $W\times P(\mathcal{B})$. The mapping of the topological hole from $W\times P(\mathcal{B})$ to $S_\mathcal{W}\times S_\mathcal{B}$ also results in a topological hole; thus, the corresponding paths cannot belong to the same simplicial cell.
\end{proof}

On the robot structural simplex, the highest order simplex is the $1$D chains, either open or closed. For example, a humanoid robot can have four chains sharing a common endpoint. These chains can uniquely determine the robot's topology and geometry. It is sufficient to plan for these chains, either independently or simultaneously. 

\newpage

\begin{figure}[t]
    \centering
    \subfigure[Robot simplex]{\includegraphics[width=2.2in]{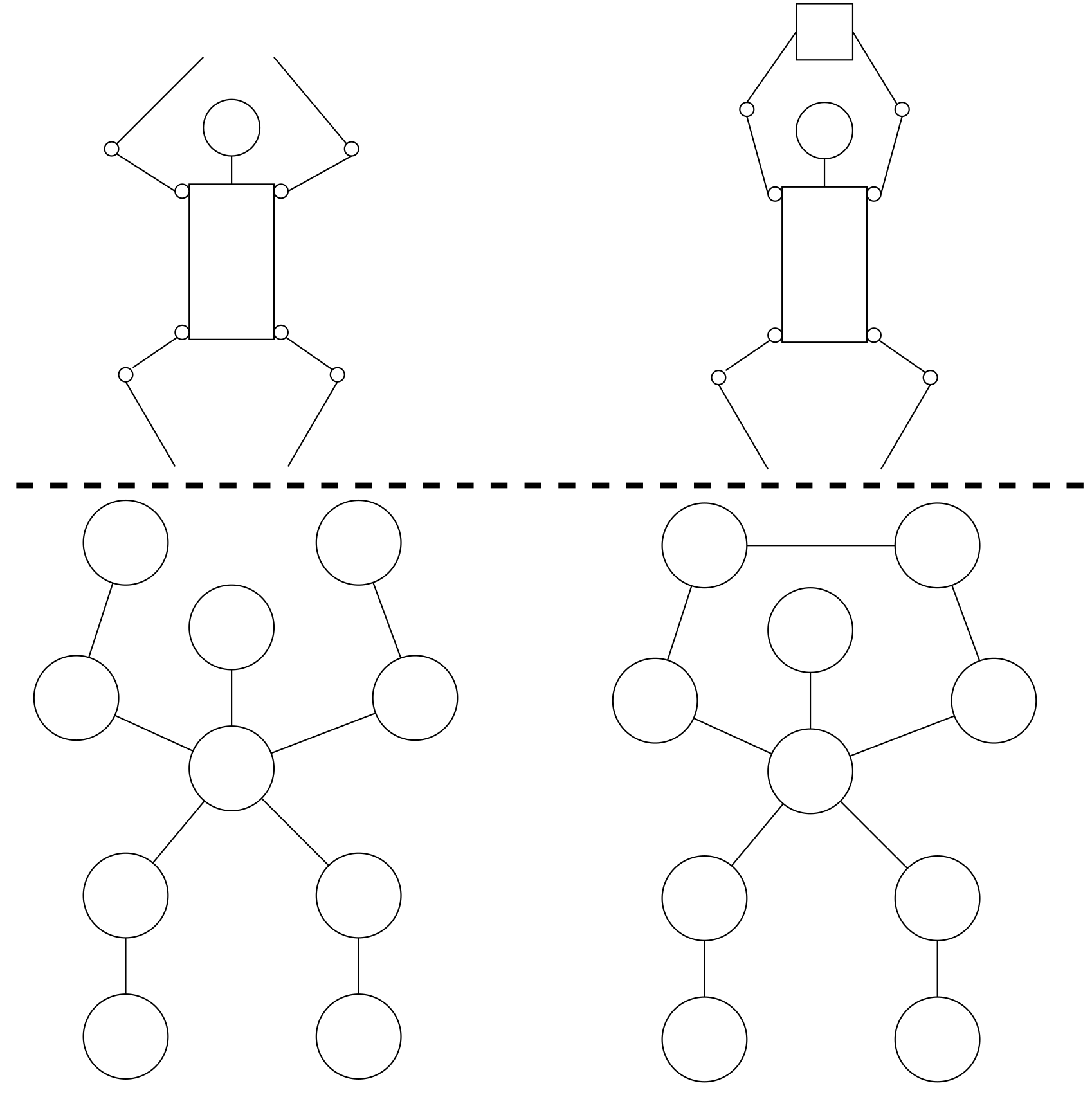}\label{fig:robot-graph}}
    \subfigure[Space simplex]{\includegraphics[width=2.2in]{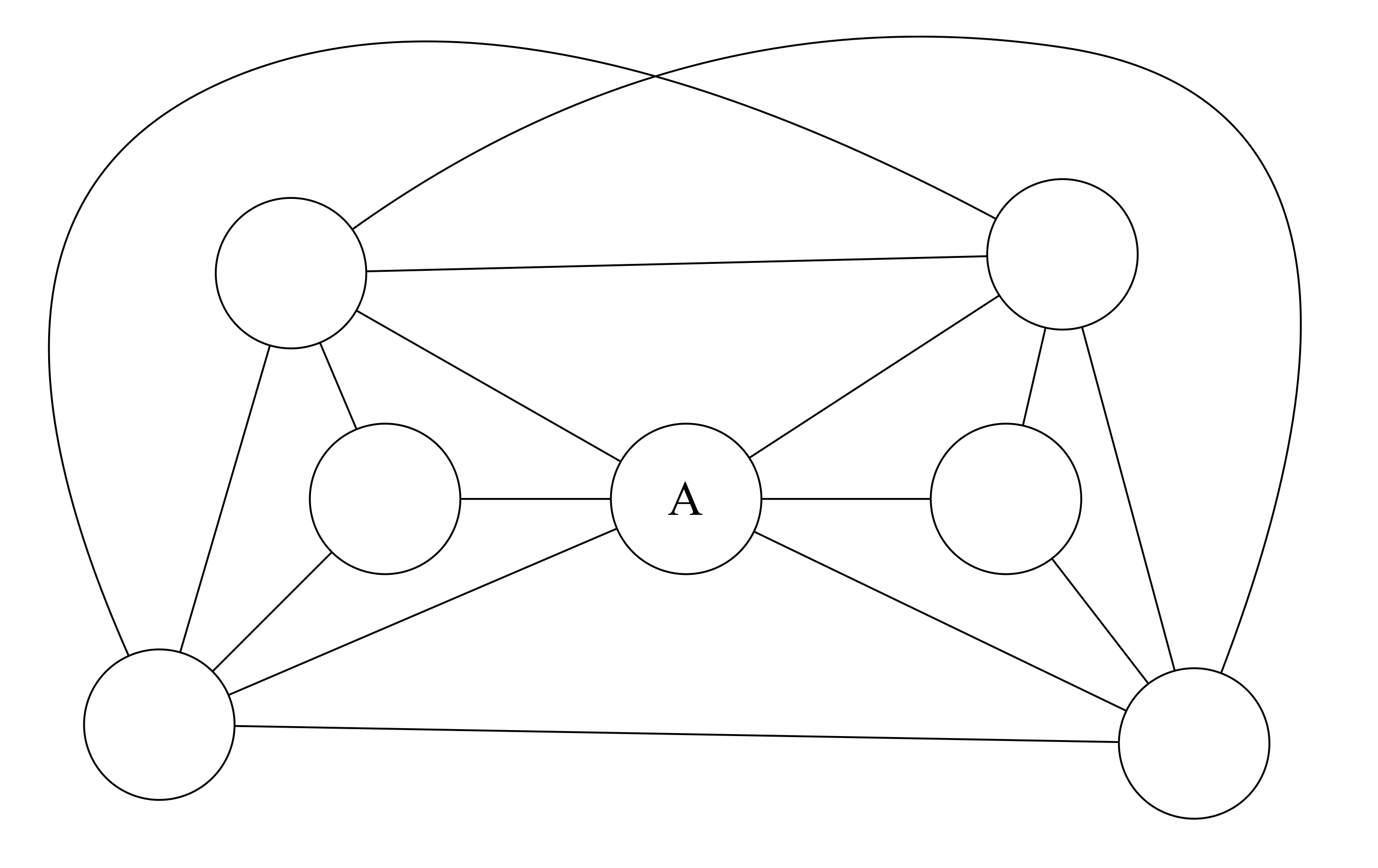}\label{fig:space-graph}}
    \caption{Sample $S_\mathcal{B}$ and $G_\mathcal{A}$. The robot simplex can change based on the current deformation, but the basic structure usually stays the same. The $G_\mathcal{A}$ is for the environment shown in Figure~\ref{fig:env-3d}. }
    \label{fig:simplex}\vspace{0.2in}
\end{figure}

In Figure~\ref{fig:simplex}, we show the $S_\mathcal{B}$ on the left, and $G_\mathcal{A}$ for environment shown in Figure~\ref{fig:env-3d} on the right. Most robots have a topology homomorphic to a collection of chains. The chain can be open or become closed based on the deformation of the robot. For example, with a humanoid robot, when the two hands hold an object together, the topology of the robot change, where the two open chains connect and form a closed chain. When considering embedding the robot into the workspace, a closed chain has to be embedded and planned as a rigid body, while the open chains can be placed into any connected free regions.

The workspace $G_\mathcal{A}$ shown in Figure~\ref{fig:space-graph} cannot be drawn as a planar graph. Each node represents an obstacle, and each edge indicates the existence of a common region between two obstacles. A triangle on $G_\mathcal{A}$ represents there is some free region adjacent to all three obstacles. Simplex of order $k$ represents a common region among $k-1$ obstacles. The $S_\mathcal{W}$ can be constructed from $G_\mathcal{A}$ so that every $k$-order simplex on $S_\mathcal{W}$ is a $k+1$th-order simplex on $G_\mathcal{A}$. 

We can even study the topology and connectivity when removing an obstacle in the space using $G_\mathcal{A}$ and $S_\mathcal{W}$. For example, in Figure~\ref{fig:space-graph}, when we treat an obstacle ($A$) in the middle of the room as a movable object and attempt to relocate the obstacle, the resulting workspace topology can be easily derived from $G_\mathcal{A}$ by removing $A$ from the graph. This task-level information can be beneficial when planning in tight spaces where interaction with the environment is needed to admit feasible paths. 

Let us define a {\em state} of the robot as a collection of pairs $L(\cdot)=(w_i, q_j)$, where $w_i$ is a vertex on $S_\mathcal{W}$ and $q_j$ a vertex on $S_\mathcal{B}$. A path $P$ is a sequence of pairs $P = \{L(1), L(2), \ldots\}$, which is an abstraction of configurations along the path. A homomorphism exists between the workspace path of a chain and the corresponding path in the configuration space. 

\newpage

\section{Path Classes}

One major challenge in planning in the configuration space is that it is difficult to map the geometry of workspace objects into the configuration space, i.e., the $\pi^{-1}$ is hard to find near different workspace neighborhoods. The workspace's geometry is often tangled with the robot topology, leading to different descriptions in the configuration space. We show that the product topology of workspace cells and robot structure leads to important topological properties in the configuration space. We then show that the proposed path classification criterion will lead to a {\em finer} path class than the homotopy class. 

The joint cover categorizes free regions based on obstacle adjacency relations, or in other words, the Voronoi cells among the obstacles. Voronoi cells describe the connectivity information in the workspace. A path in $\mathcal{C}$ cannot be valid if the mapping in $W$ is not valid. We will use workspace topology and the topology of the robot to approximate the topology in the configuration space. The configuration space topology can help find feasible paths, separate paths, and explore possible interactions with the workspace to satisfy task-level requirements. First, we show that the joint hull can help distinguish paths topologically with respect to the obstacles. 

\begin{lemma}
Given a workspace of dimension $d$, and obstacles $\mathcal{O}_i$, $0 < i \leq N$, the free-regions creates $(d-1)$-dimensional holes in the joint cover.
\label{lemma:sch_holes}
\end{lemma}
% we can do the same for 2D workspace. 

\begin{proof}
The free-regions create $d-1$-dimensional holes in the joint cover as they cut each joint hull into isolated regions. 
\end{proof}

The obstacles do not necessarily create holes in the workspace or configuration space. The statement is only valid within the joint cover. Any path intersecting the joint cover can be identified by its winding relations with different obstacles.

% define the path classification
Two paths represented by sequences of states $L(\cdot)$ without loops are geometrically near each other if a one-to-one correspondence exists between the states on the respected paths. When there are loops on a given path, we derive its embedding after removing loops. Two different labels can represent the same state, such as an open-chain deforms in the same set of connected free regions can all correspond to the same state, as long as the two endpoints of the chain remain in the same region. In 2D workspace, $S_\mathcal{W}$ has at most $1$-simplex, and in 3D, $S_\mathcal{W}$ has at most $2$-simplex. We can use Algorithm~\ref{algorithm:contract} to simplify the representation for any robot state $L(\cdot)$, and we consider two states to be different if they do not have identical contracted representations $L^c(\cdot)$. For an open chain, the simplified representation can be contracted to a single pair representing the sequence of workspace free-regions. We can similarly use G\"{o}del numbering to determine the embedding uniquely.

\begin{algorithm}[h]
\SetAlgoLined
\SetKwInOut{Input}{input}\SetKwInOut{Output}{output}
\Input{$L(i)$, $S_\mathcal{W}$, $S_\mathcal{B}$}
\Output{$L^c(i)$}
Contract $1$-simplex on $S_\mathcal{B}$ embedded in same $w_i$\;
% \For{Each $k$-simplex on $S_\mathcal{W}$ for $k>1$}{
\For{Each remaining simplex on $S_\mathcal{W}$}{
\uIf{Corresponding region is compact}{
    \If{The corresponding key-points form same simplex on $S_\mathcal{B}$}{
        Contract the simplex on $S_\mathcal{B}$ and $S_\mathcal{W}$\;
    }
}
}
return the contracted $L^c(i)$\;
 \caption{Representation contraction}
 \label{algorithm:contract}
\end{algorithm}

Let us define the geometrical-homotopy path class as follows. Given two paths $f_1$ and $f_2$ between $x$ and $y$, if the simplified representation sequences for $f_1$ and $f_2$ are different, then the two paths belong to two classes. 

% % single point, finer; 
\begin{lemma}
For point robot, the proposed path class definition is {\em finer} than the homotopy class definition. 
\label{lemma:point_finer}
\end{lemma}

\begin{proof}
If the robot is just a single point, the representation $L(\cdot)$ can be simplified to just a vertex in $S_\mathcal{W}$. The path classes and states depend only on $S_\mathcal{W}$. In 2D, the workspace partition yields the same path classes as the homotopy class for any start and goal. In 3D, two regions with different labels may still be connected. If two regions are disconnected, they must have different labels, i.e., representations. Disconnection is required to produce different homotopy path classes. 
\end{proof}

Lemma~\ref{lemma:point_finer} directly follows from Lemma~\ref{lemma:sch_holes}. The classification is {\em finer} only in dimensions three or higher. In 2D, paths passing the obstacles differently belong to different homotopy classes as the 2D obstacles create topological holes. In 3D, the paths that go through different free regions may belong to the same homotopy class, but our criterion can identify and classify them into finer categories. 

Such classification is most beneficial when narrow passages exist or the obstacle has disproportional sizes in different dimensions, such as a nail shape obstacle. In these situations, the obstacles may not create topological holes in $\mathcal{C}$, but passing different surfaces of the obstacles can result in geometrically distanced paths. Using the proposed classification criterion, paths that pass through different surfaces of obstacles may be classified into different classes, especially when the surfaces create sharp edges. 

\begin{lemma}
For any given configuration $q$ with representation $L(i)$, a compact neighborhood $U$ exists in the configuration space with the same $L(i)$. 
\end{lemma}

\begin{proof}
Small perturbations of the key points without leaving the embedding free regions have the same representations. Small perturbations of key points will lead to similar configurations.
\end{proof}

\begin{lemma}
Every region of the configuration space with the same contracted representation $L^c(i)$ is contractable, so are their finite intersections.
\label{lemma:contract_config}
\end{lemma}

\begin{proof}
Every state is homeomorphic to the product of finitely many simply connected regions in the workspace.
\end{proof}

\begin{lemma}
Given two configurations $p$ and $q$, and the corresponding representations $L^c_p(\cdot)$ and $L^c_q(\cdot)$, if loops based at arbitrary configuration $x$ passing through $p$ and passing through $q$ belong to two homotopy classes, then $L^c_p(\cdot)$ and $L^c_q(\cdot)$ must be different. 
\label{lemma:rep_equal}
\end{lemma}

\begin{proof}
For a given point $x$, if the loop passing through $p$ and $q$ belong to different homotopy classes, obstacles exist that prevent the loop from passing through $p$ to deform to $q$ in the configuration space. If $L(p)$ and $L(q)$ have the same representation, each key point must be in the same embedding region following the preceding Lemma. Therefore, no obstacles prevent these key points from deforming between $p$ and $q$. The omitted key points must be along a fixed path passing through a connected sequence of regions, thus can deform between $p$ and $q$. Therefore, two configurations $p$ and $q$ are connected in the same free regions if they have the same representation.  
\end{proof}

\newpage

\begin{theorem}
Given two paths in $\mathcal{C}$, if two paths belong to different homotopy classes, then the two paths must have at least one different representation. 
\label{theorem:diff_path_rep}
\end{theorem}

\begin{proof}
If two paths are in different homotopy classes, obstacles must exist preventing the paths from continuously deforming to each other. Then, there must exist obstacles that separate some states along the paths. Different states have different representations. 
\end{proof}

\begin{corollary}
Any two paths that belong to two different homotopy classes on $\mathcal{C}$ maps to two paths in different classes defined by the proposed contracted representations on $S_\mathcal{W}\times S_\mathcal{B}$. 
\end{corollary}

\begin{lemma}
Two paths in $\mathcal{C}$ may have different contracted representations $L^c_p(\cdot)$ but belong to the same homotopy class on $\mathcal{C}$. 
\end{lemma}

\begin{proof}
If the robot is a point robot, the statement follows from Lemma~\ref{lemma:point_finer}. If the robot is not a point robot, such as an open chain, two paths that go around an obstacle near different surfaces can have different representations, as the paths go through different cells in $S_\mathcal{W}$. However, if the two surfaces are not disjoint, the two paths will belong to the same homotopy class. 
\end{proof}

The above lemmas and theorem show that workspace regions and robot structure representation capture the topology information in the configuration space. Therefore, two configuration space paths in different homotopy classes must pass through configurations with non-identical states. However, non-identical states may have the same fundamental group in the configuration space. Therefore, we have,

\begin{theorem}
The proposed path class definition is {\em finer} than the homotopy class. 
\label{theorem:config_finer}
\end{theorem}

\begin{proof}
The theorem is a direct extension of the proceeding Lemmas and Theorem. 
\end{proof}

% \subsection*{Path planning framework}

We present the following planning framework based on the path classes defined in the preceding sections. The framework finds topological paths globally and then interpolates the feasible geometries for the robots to implement. This framework aims to separate global navigation from local geometrical implementation. The separation can allow adjustments of local paths without rescheduling the global plan. 

\begin{enumerate}
    \item Find path classes, and adjacency relations of different free regions; 
    \item Find paths $f$ on $S_\mathcal{W}\times S_\mathcal{B}$;
    \item Find point path $f_p$ that realizes $f$ in $W$;
    \item Divide $f_p$ into pieces that are simple and mostly visible;
    \item Interpolate geometrical trajectories piece-wise on $f_p$, compute feasible robot deformations; 
    \item Adjust paths piece-wise to satisfy robot geometrical constraints; 
    \item Optimize and smooth the trajectory if needed. 
\end{enumerate}

The proposed framework is most beneficial when there are narrow passages in the space. The computation needed to produce the path classes can exceed that of sampling-based motion planning in a relatively open space. As the environment and robot become complex, the proposed framework can produce paths according based on tasks and allow manipulations and collaborations in tight spaces. 

\section{Applications and Verification}

First, in Figure~\ref{fig:env-3d}, we show the environment with multiple obstacles, as our proposed path class mainly differs from the homotopy class in 3D. The obstacles are not creating topological holes, and all paths belong to the same homotopy class. However, our definition would recognize additional classes. The joint cover divides the free regions among adjacent obstacle pairs uniquely. A path passing through different free regions in the joint hulls creates different path classes for a single point. 
% Therefore, the proposed path classes are finer than the homotopy classes. 

One primary aim of this paper is to show that the analysis of the workspace along and the robot structural information is sufficient to capture topological information in the configuration space. The resulting path class information can derive much helpful information that usually is extracted using various methods. 

For example, in~\citet{mccarthy2012proving}, the users used the sampling-based method to compute $\alpha$-shapes of configuration space obstacles and prove the non-existence of paths. In the proposed method, the structure $S_\mathcal{W}\times S_\mathcal{B}$ can derive the same information. The idea is very similar to the process used in~\citet{schwartz1983piano}. We can follow the below steps, 
\begin{itemize}
    \item If there does not exist a sequence on $S_\mathcal{W}$ that connects start and goal representation regions, no path would exist;
    \item If the sequence on $S_\mathcal{W}$ contains regions that cannot have valid robot embedding, remove this sequence;
    \item If no sequence remains, there does not exist a path. 
\end{itemize}

\begin{figure*}[h]
\begin{center}
\mbox{
    % \hspace{-0.2in}
    % \usebox{\imageboxRoom}
    % \hspace{-0.1in}
    % \raisebox{\dimexpr.5\ht\imageboxRoom-.47\height}{% Raise smaller image into place
    % \includegraphics[width=1.5in]{figs/room-env-free-region-1.pdf}}
    % \raisebox{\dimexpr.5\ht\imagebox-.5\height}{% Raise smaller image into place
    %   \subfigure[Random gray polygons and corresponding joint cover obtained using Algorithm~\ref{algorithm:joint}. ]{
    %   \includegraphics[height=1.1in,width=1.55in]{figs/env_random_1.png}
    %   }
    % %   \hspace{-2em}
    % %   \raisebox{\dimexpr.5\ht\imagebox-.5\height}{% Raise smaller image into place
    %   \subfigure[Random gray polygons and corresponding joint cover obtained using Algorithm~\ref{algorithm:joint}. ]{
    %   \includegraphics[height=1.1in,width=1.55in]{figs/env_random_2.png}
    %   }
    % \hspace{-0.1in}
    % \raisebox{\dimexpr.5\ht\imageboxRoom-.42\height}{
    \subfigure[Environment with three boxes. ]{
    \label{fig:env-arm-box}
    \includegraphics[height=1.7in]{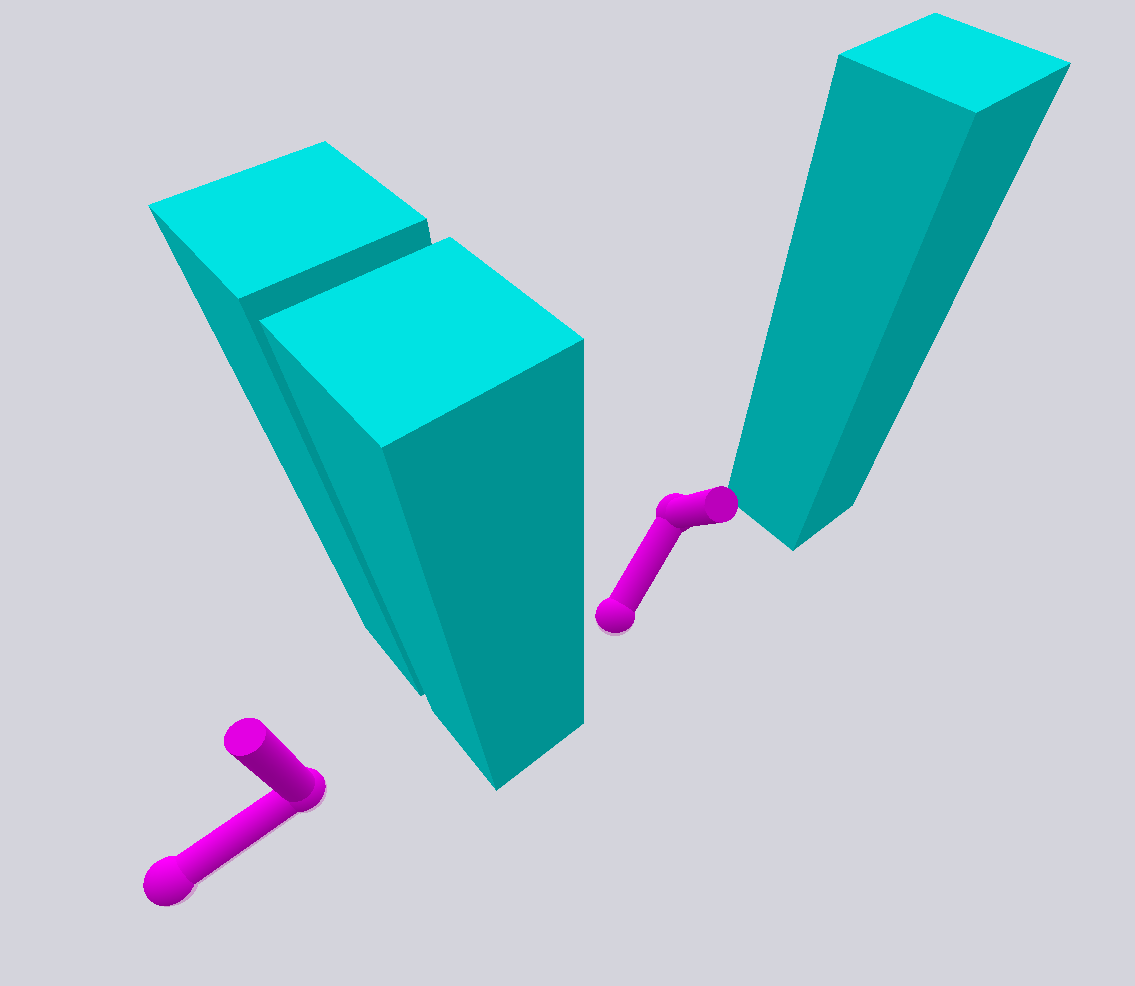}
    }
    \subfigure[Different path classes for an arm traced along the base.]{
    \label{fig:env-paths}
    \includegraphics[height=1.7in]{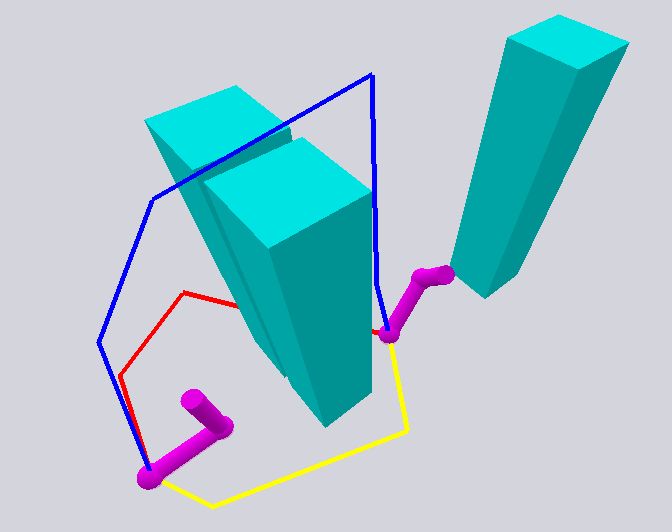}
    }
    % \hspace{-0.1in}
    % \raisebox{\dimexpr.5\ht\imageboxRoom-.42\height}{
    \subfigure[Planning a 4R arm in an environment with narrow passages.]{
    \label{fig:env_narrow}
    \includegraphics[width=1.9in]{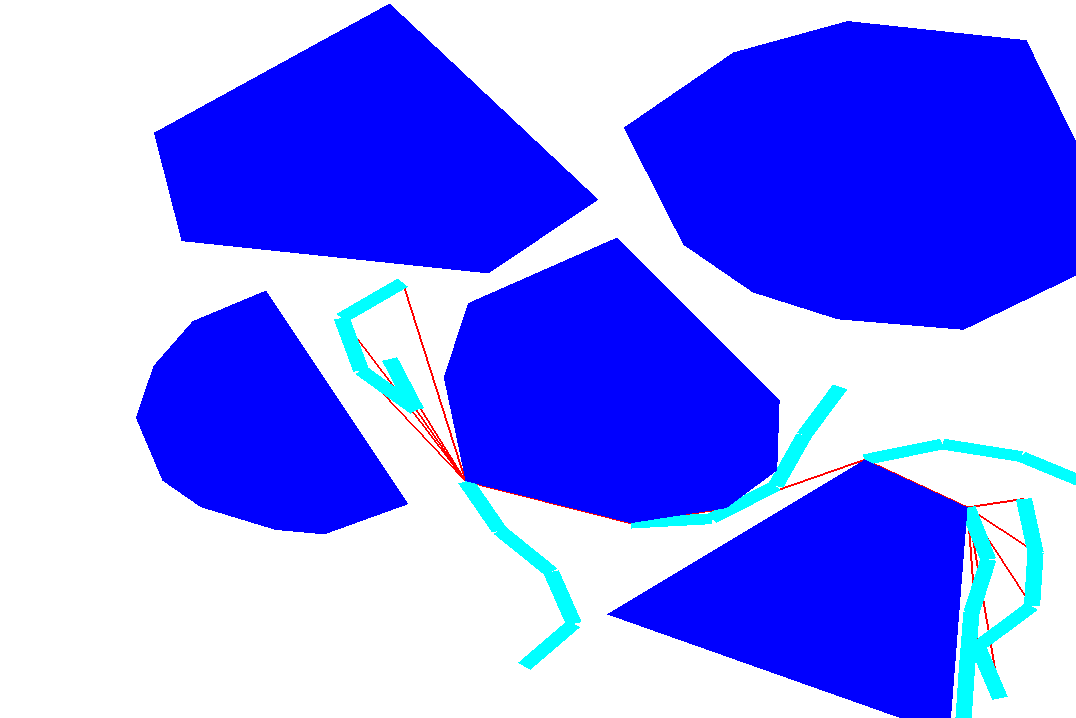}
    }
}
\end{center}
    \caption{Joint covers of different environments, path classes, and narrow passage planning scenario. }
    \label{fig:room_env_region}
\end{figure*}

In Figure~\ref{fig:env-arm-box}, two boxes have very close distances, but the pink arm cannot move through the narrow region between the boxes. Configuration space narrow passage among the two boxes can be hard to detect, but the free-region derived in the workspace can be directly used to test against the robot's geometry to determine if a path exists through some region. In the environment shown in Figure~\ref{fig:env-arm-box}, all paths of the arm belong to the same homotopy class, but our criterion will show that there three colored paths shown in Figure~\ref{fig:env-paths} are different, and indeed, the red and the yellow paths are pretty distanced geometrically. The criterion can have more advantages when the robot becomes more complex.

The proposed planning framework further utilizes such workspace information. By considering only the key points of the robot, we can quickly check the feasibility of the free regions with respect to the size of the key points. Though the paths for key points do not necessarily mean the existence of paths for the entire robot, the counterexamples would usually include narrow passages where paths are difficult to find. For example, in Figure~\ref{fig:env_narrow}, the paths for a four-R arm are found using the proposed framework, where the sampling-based planning approach failed to find a solution after $4000$ valid samples. In the workspace, we can directly test the valid embedding sequence against the robot geometry and workspace constraints, and deformations of robots can be derived to find valid configurations to fit through the narrow passages. 

We can also use the path classes to find narrow passages or paths with the most chance for survival against probabilistic events~\citep{erickson2009survivability,lyu2016k}. The robot structural information in path classes may help validate paths for multi-agent planning or finding particular configurations to fit through narrow passages, such as the planning of the twistycool problem in OMPL~\citep{sucan2012open}. 

The proposed path classes are distant geometrically. The resulting path classes are good candidates for multiple agents to avoid collisions. The distance path classes can also allow collaborations in tight spaces. We can further introduce pair-wise distance constraints among key points of different robots in the optimization step, coordinating the paths for the collaborating robots to maintain relative distances and avoid collisions. Such problems can be solved using the proposed planning framework.

\section{Conclusions}

This work presents a path class definition that combines the workspace geometry and the topology of the robot. The resulting path class is finer than the homotopy path classes. The classification captures geometric and topological information in the configuration space and can help identify connectivity, obstacle topology, and compare path qualities in different path classes. Building upon the presented path classes and the algorithms to find them, we propose a new planning framework that separates the global topological plan from local geometrical interpolation. The planning framework, like the proposed path classes, builds upon the workspace's geometrical analysis. The details of the geometrical implementation of the topological plan are ongoing work. The presented work gives us theoretical foundations for the correctness of separated analysis of workspace geometry and the topology of the robot and then combines them to piece together trajectories in configuration space. One important future work is to reduce further the size and complexity of $S_\mathcal{W}\times S_\mathcal{B}$. Currently, the map $\pi$ from $\mathcal{C}$ to $\mathcal{W}$ is not being carefully studied. We believe further analysis can be introduced to the map $\pi$ to produce simpler $S_\mathcal{B}$ near different positions in $\mathcal{W}$. 

%Some initial results have shown advantages near narrow passages and are currently under review.

\bibliographystyle{plainnat}
\bibliography{refs}

\end{document}